\documentclass{amsart}

\usepackage{amssymb}
\usepackage{graphicx}
\usepackage{url}

\newtheorem{theorem}{Theorem}[section]
\newtheorem{lemma}[theorem]{Lemma}
\newtheorem{proposition}[theorem]{Proposition}
\newtheorem{corollary}[theorem]{Corollary}

\theoremstyle{remark}
\newtheorem{remark}[theorem]{Remark}

\numberwithin{equation}{section}

\def\R{\mathbb{R}}

\def\eps{\varepsilon}
\def\dd{\mathop{}\!d}
\def\half{\tfrac12}
\def\norm#1{\lVert#1\rVert}
\def\abs#1{\lvert#1\rvert}
\def\Ltwo{L^{2}((0,e_0))}
\def\Hone{H^{1}((0,e_0))}

\begin{document}

\title{Least-time gradient flow}

\author[Alessandro Betti]{Alessandro Betti$^{1}$}
\author[Marco Gori]{Marco Gori$^{2}$}
\author[Stefano Melacci]{Stefano Melacci$^{2}$}
\author[Jinwei Zhao]{Jinwei Zhao$^{2,3}$}
\thanks{$^{1}$Department of Mathematics ``Tullio Levi-Civita'',
  University of Padua, Italy.}
\thanks{$^{2}$Department of Information Engineering and Mathematics,
  University of Siena, Italy.}
\thanks{$^{3}$University of Florence, Italy.}
\thanks{\emph{E-mail addresses}: \texttt{alessandro.betti@unipd.it} (A.~Betti),
  \texttt{marco.gori@unisi.it} (M.~Gori),
  \texttt{stefano.melacci@unisi.it} (S.~Melacci),
  \texttt{jinwei.zhao@unifi.it} (J.~Zhao).}

\date{\today}

\begin{abstract}
  Prescribing the speed of gradient flow on the risk itself, by the
  dynamics $\dot w=-u(E(w))\nabla E(w)/\abs{\nabla E(w)}^{2}$, makes the
  risk $e(t)=E(w(t))$ obey $\dot e=-u(e)$ exactly, whatever the
  landscape~$E$; the time needed to reach zero risk from $e_0$ is
  $\int_0^{e_0}\dd e/u(e)$. Minimizing this time alone is ill posed, and
  we study the regularized problem
  $\inf\{\int_0^{e_0}(\tfrac\lambda2\abs{u'}^{2}+1/u)\,\dd e:\
  u\in H^{1}(0,e_0),\ u\ge0,\ u(0)=0\}$, $\lambda>0$. We prove that the
  minimizer exists, is unique, and is a linearly scaled cycloid, and we
  show that the optimal rate behaves like
  $u^{*}(e)\sim(9/(2\lambda))^{1/3}e^{2/3}$ near zero risk: the exponent
  $2/3$ is the one found in \cite{betti2026holder} by a power-law
  ansatz, and it lies in the H\"older window $(\tfrac12,1)$ where the
  arrival is in finite time with vanishing weight speed. The proof
  follows the classical route: existence by the direct method,
  uniqueness by strict convexity, positivity of the minimizer away from
  the origin, and the explicit integration of the Euler--Lagrange
  equation.
\end{abstract}

\maketitle

\section{Introduction}

Gradient flow $\dot w(t)=-\nabla E(w(t))$ on the quadratic risk
$E(w)=\half\abs{w}^{2}$ makes the risk $e(t):=E(w(t))$ obey
$\dot e=-\abs{\nabla E(w)}^{2}=-2e$, so that $e(t)=e_0e^{-2t}$: the
risk never reaches zero. This is not a peculiarity of the quadratic
risk but a property of every risk with a locally Lipschitz gradient:
the trajectories of a Lipschitz vector field are unique in both time
directions \cite[Chapter~1, Theorem~2.2]{coddington1955theory}, and
since a critical point of $E$ is itself a trajectory, no other
trajectory can reach it in finite time. A direct cure is to divide the
field by $\abs{\nabla E}^{2}$: the dynamics
$\dot w=-\nabla E/\abs{\nabla E}^{2}$ gives $\dot e=-1$ and
$e(t)=e_0-t$, zero at $t=e_0$ on every landscape. But on the quadratic
risk this dynamics is $\dot w=-w/\abs{w}^{2}$, so that
$\abs{w(t)}=\sqrt{\abs{w_0}^{2}-2t}$ and
\[
  \abs{\dot w(t)}=\frac1{\abs{w(t)}}=\frac1{\sqrt{2\,(e_0-t)}}
  \ \longrightarrow\ +\infty\qquad\hbox{as }t\to e_0\,:
\]
the risk reaches zero in finite time, but the weights arrive at the
minimum with infinite speed. Both dynamics belong to the family
\begin{equation}\label{eq:dynamics}
  \dot w(t)=-\,\frac{u\bigl(E(w(t))\bigr)}{\abs{\nabla E(w(t))}^{2}}\,
  \nabla E(w(t))\,,
\end{equation}
studied in~\cite{betti2026holder}, in which the speed is set by a
function $u\colon[0,+\infty)\to[0,+\infty)$ of the current risk; we call
$u$ the \emph{dissipation schedule}. The two examples above are $u\equiv1$
and, on the quadratic risk, $u(e)=2e$; the choice $u(e)=e$ is Polyak's
step \cite{polyak1987} with optimal value zero. The dynamics is defined
wherever $\nabla E\ne0$, which in the interpolation regime of
overparameterized networks is exactly where $E>0$ (this standing
assumption, $\nabla E=0\iff E=0$, and the well-posedness of
\eqref{eq:dynamics} are discussed in~\cite{betti2026holder}; the present
note is about the schedule only). What makes \eqref{eq:dynamics}
particularly interesting is the following relation, a consequence of
the chain rule:
\begin{equation}\label{eq:reduction}
  \dot e=\langle\nabla E,\dot w\rangle
  =-\,\frac{u(e)}{\abs{\nabla E}^{2}}\,\langle\nabla E,\nabla E\rangle
  =-\,u(e)\,,
\end{equation}
\emph{on every landscape}. The geometry of $E$ decides the path of the
weights; $u$ decides only the clock along it, and the clock is explicit:
separating the variables in $\dot e=-u(e)$, the time needed to bring the
risk from $e_0$ down to zero is
\begin{equation}\label{eq:time}
  T(u)=\int_0^{e_0}\frac{\dd e}{u(e)}\,.
\end{equation}
It is finite for $u\equiv1$ ($T=e_0$) and infinite for $u(e)=e$ (the
exponential decay of gradient flow and of Polyak's step); in the language
of \cite{zak1988,zak1989,bianchini1997}, $T(u)<\infty$ means that the
minimum is a \emph{terminal attractor} of the risk dynamics, and the
criterion ``$\int\dd e/u<\infty$'' is the one-dimensional case of the
finite-time stability theory of \cite{bhat2000}; see also
\cite{cortes2006} for finite-time gradient flows.

The schedule of least time does not exist: constants $u\equiv c$ give
$T=e_0/c$, which tends to zero as $c\to\infty$, and since
$\abs{\dot w}=c/\abs{\nabla E}$, every such schedule arrives with
infinite speed. What we want is the fastest schedule that lands
\emph{softly}, and the way to ask for it is to require $u(0)=0$, so that
the rate itself vanishes at arrival, and to pay for roughness: we
minimize
\begin{equation}\label{eq:F}
  F(u)=\int_0^{e_0}\Bigl(\frac\lambda2\,\abs{u'(e)}^{2}+\frac1{u(e)}\Bigr)\dd e\,,
  \qquad\lambda>0\,,
\end{equation}
over the class
\begin{equation}\label{eq:X}
  X=\bigl\{u\in\Hone:\ u\ge0\ \hbox{on }[0,e_0],\ u(0)=0\bigr\}\,.
\end{equation}
Functions in $\Hone$ are identified with their continuous
representatives on $[0,e_0]$, so that the pointwise conditions in
\eqref{eq:X} make sense. The potential term in \eqref{eq:F} is the
Lebesgue integral of a nonnegative function: we set $F(u)=+\infty$ when
$u$ vanishes on a set of positive measure, or when $1/u$, then defined
almost everywhere, is not integrable; otherwise $F(u)$ is given by
\eqref{eq:F}. In particular
\begin{equation}\label{eq:finite-energy}
  F(u)<+\infty\quad\Longrightarrow\quad\abs{\{u=0\}}=0\,.
\end{equation}
The second term of $F$ is the time \eqref{eq:time}; the first penalizes
abrupt changes of the rate, and $\lambda$ sets the price of roughness.
(The functional of \cite{betti2026holder} carries two weights
$\lambda_1,\lambda_2$ and an amplitude term; a single parameter suffices,
and it can even be scaled away, Remark~\ref{rem:scaling}.) The purpose
of this note is to solve this problem completely. Its answer was
announced in \cite{betti2026holder}, where the exponent $2/3$ was
obtained from the Euler--Lagrange equation by a power-law ansatz; here
is the full statement.

\begin{theorem}\label{thm:main}
  Let $e_0>0$ and $\lambda>0$. The functional $F$ of \eqref{eq:F} has
  exactly one minimizer $u^{*}$ on $X$, and $u^{*}$ is the function
  of~$e$ defined parametrically by
  \begin{equation}\label{eq:cycloid}
    e(\theta)=\frac{2e_0}{\pi}\,(\theta-\sin\theta\cos\theta)\,,\quad
    u^{*}(\theta)=\lambda^{-1/3}\Bigl(\frac{2\sqrt2\,e_0}{\pi}\Bigr)^{\!2/3}
    \sin^{2}\theta\,,
    \quad 0\le\theta\le\frac\pi2\,,
  \end{equation}
  a linearly scaled cycloid. It is continuous and strictly increasing on
  $[0,e_0]$, of class $C^{\infty}$ on $(0,e_0]$, and it solves
  \begin{equation}\label{eq:EL}
    \lambda\,u''(e)=-\,\frac1{u(e)^{2}}\quad\hbox{for every }e\in(0,e_0]\,,
    \qquad u(0)=0\,,\qquad u'(e_0)=0\,.
  \end{equation}
\end{theorem}

\begin{remark}[on the parametrization]\label{rem:param}
  The map $\theta\mapsto e(\theta)$ in \eqref{eq:cycloid} is a strictly
  increasing bijection of $[0,\pi/2]$ onto $[0,e_0]$
  (Proposition~\ref{prop:cycloid}); so \eqref{eq:cycloid} defines $u^{*}$
  as a function on $[0,e_0]$, the value $u^{*}(e)$ being the second
  expression evaluated at the unique $\theta=\theta(e)$ with
  $e(\theta)=e$. To keep the two roles of the letters apart, in the
  proofs we write $\varepsilon(\theta)$ and $U(\theta)$ for the two
  functions of $\theta$ in \eqref{eq:cycloid}, so that
  $u^{*}=U\circ\varepsilon^{-1}$ and $\theta(e)=\varepsilon^{-1}(e)$.
\end{remark}

\begin{corollary}[the cusp]\label{cor:cusp}
  $u^{*}(e)=(9/(2\lambda))^{1/3}\,e^{2/3}\,\bigl(1+O(e^{2/3})\bigr)$ as
  $e\to0$.
\end{corollary}

\begin{corollary}[minimum value and arrival time]\label{cor:values}
  \[
    F(u^{*})=\frac32\,\lambda^{1/3}(\pi^{2}e_0)^{1/3}\,,\qquad
    T(u^{*})=\lambda^{1/3}(\pi^{2}e_0)^{1/3}\,,
  \]
  so that the penalty $\frac\lambda2\int\abs{u^{*\prime}}^{2}$ is exactly
  half of the time.
\end{corollary}

\begin{figure}[b]
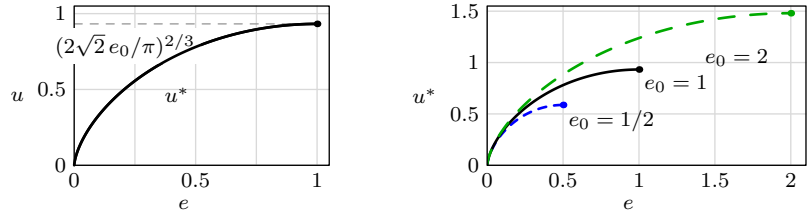

  \hbox to\hsize{\hfil\includegraphics{figs/uopt-1.mps}\hfil
    \includegraphics{figs/uopt-2.mps}\hfil}
  \caption{The minimizer \eqref{eq:cycloid} for $\lambda=1$ and $e_0=1$,
    with its maximum $A=(2\sqrt2\,e_0/\pi)^{2/3}$ marked (left), and the
    family for different initializations $e_0=\frac12,1,2$ (right).}
  \label{fig:cycloid}
\end{figure}

Figure~\ref{fig:cycloid} shows the minimizer. The plan is the following.
Section~\ref{sec:cycloid} derives the candidate \eqref{eq:cycloid} from
the Euler--Lagrange equation, as one does in the indirect method, and
verifies that it is an admissible competitor of finite energy.
Section~\ref{sec:proof} proves the theorem along the classical route:
existence by the direct method, uniqueness by strict convexity, then
positivity of the minimizer away from the origin, which makes the
Euler--Lagrange equation available, and finally its explicit
integration, which identifies the minimizer with \eqref{eq:cycloid}.
Section~\ref{sec:holder} checks that the minimizer is consistent with
the H\"older regime of \cite{betti2026holder}: it lands inside the
window $\tfrac12<\alpha<1$, at its exponent $2/3$.

\section{The extremal: a cycloid}\label{sec:cycloid}

Let us compute, in the spirit of the indirect method of the calculus of
variations, the Euler--Lagrange equation of $F$. For a smooth positive
$u$ and a smooth $\varphi$ with $\varphi(0)=0$, differentiating
$F(u+t\varphi)$ at $t=0$ and integrating by parts,
\[
  \frac{\dd}{\dd t}F(u+t\varphi)\Big|_{t=0}
  =\int_0^{e_0}\Bigl(\lambda u'\varphi'-\frac{\varphi}{u^{2}}\Bigr)\dd e
  =\lambda u'(e_0)\varphi(e_0)
  -\int_0^{e_0}\Bigl(\lambda u''+\frac1{u^{2}}\Bigr)\varphi\,\dd e\,.
\]
Since $\varphi(e_0)$ is free, a minimizer should satisfy the
Euler--Lagrange equation together with the natural boundary condition
at $e_0$, that is \eqref{eq:EL}
\cite[Chapter~1, Sections~2.2 and~2.4]{giaquinta1996}, and we are
interested in positive solutions. Multiplying $\lambda u''=-1/u^{2}$ by
$u'$ we get $\bigl(\frac\lambda2u'^{2}-1/u\bigr)'=0$, and the boundary
condition at $e_0$ fixes the constant:
\begin{equation}\label{eq:first-integral}
  \frac\lambda2\,u'^{2}=\frac1u-\frac1A\,,\qquad A:=u(e_0)\,.
\end{equation}
As $u''<0$, the derivative $u'$ is strictly decreasing, and $u'(e_0)=0$
forces $u'>0$ on $(0,e_0)$; hence $u'=\sqrt{(2/\lambda)(1/u-1/A)}$, a
separable equation, whose solution in parametric form is
\eqref{eq:cycloid}: the computation is carried out in
Section~\ref{sub:identification}, where it identifies the minimizer.
The curve is Johann Bernoulli's brachistochrone \cite{bernoulli1696}:
after the change of scale $e=\sqrt{\lambda A/2}\,x$, equation
\eqref{eq:first-integral} becomes $(\dd u/\dd x)^{2}=(A-u)/u$, the first
integral $y(1+y'^{2})=\hbox{const}$ of the bead sliding along a curve
$y(x)$ under gravity, whose solutions are the cycloids.

So much for the heuristics. What the proof will need of
\eqref{eq:cycloid} is that it defines an admissible competitor of finite
energy, its energy, and its behaviour at the cusp.

\begin{proposition}\label{prop:cycloid}
  Let $A=\lambda^{-1/3}(2\sqrt2\,e_0/\pi)^{2/3}$ and, for
  $0\le\theta\le\pi/2$, let
  \[
    \varepsilon(\theta)=\frac{2e_0}\pi\,(\theta-\sin\theta\cos\theta)\,,
    \qquad U(\theta)=A\sin^{2}\theta
  \]
  be the two functions in \eqref{eq:cycloid}. Then $\varepsilon$ is a
  $C^{\infty}$ strictly increasing bijection of $[0,\pi/2]$ onto
  $[0,e_0]$, its inverse $\theta(e):=\varepsilon^{-1}(e)$ is $C^{\infty}$
  on $(0,e_0]$, and $u^{*}=U\circ\varepsilon^{-1}$ has the following
  properties.
  \begin{itemize}
  \item[\textup{(a)}] $u^{*}\in C([0,e_0])\cap C^{\infty}((0,e_0])$,
    $u^{*}(0)=0$, $u^{*}(e_0)=A$, $u^{*}$ is strictly increasing, and
    $u^{*\prime}(e)=\kappa\cot\theta(e)$ on $(0,e_0]$, where
    $\kappa:=\pi A/(2e_0)$ satisfies $\kappa^{2}=2/(\lambda A)$.
  \item[\textup{(b)}] $u^{*}\in X$, and
    \[
      \frac\lambda2\int_0^{e_0}\abs{u^{*\prime}}^{2}\,\dd e=\frac{e_0}A\,,\qquad
      \int_0^{e_0}\frac{\dd e}{u^{*}}=\frac{2e_0}A\,,\qquad
      F(u^{*})=\frac{3e_0}A=\frac32\,\lambda^{1/3}(\pi^{2}e_0)^{1/3}\,.
    \]
  \item[\textup{(c)}] $u^{*}(e)=(9/(2\lambda))^{1/3}e^{2/3}\,(1+O(e^{2/3}))$
    as $e\to0$.
  \end{itemize}
\end{proposition}

\begin{proof}
  The function $\varepsilon$ is smooth, with $\varepsilon(0)=0$,
  $\varepsilon(\pi/2)=e_0$ and
  \begin{equation}\label{eq:eps-prime}
    \varepsilon'(\theta)=\frac{2e_0}\pi\,(1-\cos2\theta)
    =\frac{4e_0}\pi\,\sin^{2}\theta>0\quad\hbox{for }0<\theta\le\frac\pi2\,;
  \end{equation}
  so it is a homeomorphism of $[0,\pi/2]$ onto $[0,e_0]$, and by the
  inverse function theorem its inverse is $C^{\infty}$ on $(0,e_0]$, where
  $\varepsilon'\ne0$.

  (a) follows at once, since $U$ is smooth, $U(0)=0$, $U(\pi/2)=A$ and
  $U$ is strictly increasing on $[0,\pi/2]$; by the chain rule and
  \eqref{eq:eps-prime},
  \begin{equation}\label{eq:u-prime}
    u^{*\prime}(e)=\frac{U'(\theta)}{\varepsilon'(\theta)}
    =\frac{2A\sin\theta\cos\theta}{(4e_0/\pi)\sin^{2}\theta}
    =\frac{\pi A}{2e_0}\cot\theta\,,
  \end{equation}
  and $\kappa^{2}=\pi^{2}A^{2}/(4e_0^{2})=2/(\lambda A)$ because
  $\lambda\pi^{2}A^{3}=8e_0^{2}$, which is the definition of $A$ cubed.

  (b) The integrals are computed by the substitution $e=\varepsilon(\theta)$,
  which is legitimate for nonnegative continuous integrands on $(0,e_0]$
  (monotone convergence on $[\delta,e_0]$, $\delta\downarrow0$). By
  \eqref{eq:eps-prime}, \eqref{eq:u-prime} and $\kappa^{2}=2/(\lambda A)$,
  \begin{gather*}
    \frac\lambda2\int_0^{e_0}\abs{u^{*\prime}}^{2}\,\dd e
    =\frac\lambda2\int_0^{\pi/2}\kappa^{2}\cot^{2}\theta\cdot\frac{4e_0}\pi\sin^{2}\theta\,\dd\theta
    =\frac{\lambda\kappa^{2}}2\cdot\frac{4e_0}\pi\cdot\frac\pi4=\frac{e_0}A\,,\\
    \int_0^{e_0}\frac{\dd e}{u^{*}}
    =\int_0^{\pi/2}\frac{(4e_0/\pi)\sin^{2}\theta}{A\sin^{2}\theta}\,\dd\theta
    =\frac{4e_0}{\pi A}\cdot\frac\pi2=\frac{2e_0}A\,.
  \end{gather*}
  The first computation shows that $u^{*\prime}$ is square integrable;
  as $u^{*}$ is continuous on $[0,e_0]$ and $C^{1}$ on $(0,e_0]$, it
  follows that $u^{*}\in\Hone$ with weak derivative $u^{*\prime}$
  (letting $\delta\downarrow0$ in $u^{*}(x)-u^{*}(\delta)=\int_\delta^xu^{*\prime}$
  gives $u^{*}(x)=\int_0^xu^{*\prime}$). With (a), $u^{*}\in X$; its zero
  set is $\{0\}$, and the second computation shows that $1/u^{*}$ is
  integrable. Adding the two, $F(u^{*})=3e_0/A$; since
  $A=\lambda^{-1/3}\cdot2(e_0/\pi)^{2/3}$, this is
  $\frac32\lambda^{1/3}\pi^{2/3}e_0^{1/3}$.

  (c) As $\theta\to0$, $\sin2\theta=2\theta-\frac43\theta^{3}+O(\theta^{5})$,
  so
  \[
    \varepsilon(\theta)=\frac{2e_0}\pi\Bigl(\theta-\frac12\sin2\theta\Bigr)
    =\frac{4e_0}{3\pi}\,\theta^{3}\bigl(1+O(\theta^{2})\bigr)\,,\qquad
    U(\theta)=A\theta^{2}\bigl(1+O(\theta^{2})\bigr)\,.
  \]
  Inverting the first relation, $\theta=(3\pi e/4e_0)^{1/3}(1+O(e^{2/3}))$,
  and substituting in the second,
  \begin{gather*}
    u^{*}(e)=A\Bigl(\frac{3\pi e}{4e_0}\Bigr)^{2/3}\bigl(1+O(e^{2/3})\bigr)\,,\\
    A\Bigl(\frac{3\pi}{4e_0}\Bigr)^{2/3}
    =\lambda^{-1/3}\Bigl(\frac{2\sqrt2\,e_0}\pi\cdot\frac{3\pi}{4e_0}\Bigr)^{2/3}
    =\lambda^{-1/3}\Bigl(\frac{3\sqrt2}2\Bigr)^{2/3}
    =\Bigl(\frac9{2\lambda}\Bigr)^{1/3}\,.
    \qedhere
  \end{gather*}
\end{proof}

\section{Proof of Theorem~\ref{thm:main}}\label{sec:proof}

We follow the classical route: existence by the direct method,
uniqueness by strict convexity, then the study of the minimizer---first
its monotonicity and positivity away from the origin, which make the
Euler--Lagrange equation available, then the integration of the
equation.

\subsection{Existence}\label{sub:existence}

The direct method requires the sublevel sets of $F$ to be compact for
some notion of convergence, and $F$ to be lower semicontinuous for the
same notion. The only point where this problem differs from a textbook
application is the potential term: a minimizing sequence could, in
principle, converge to a function vanishing on a set of positive
measure, where $F$ is $+\infty$ and Fatou's lemma would tell us nothing.
The following lemma says that the energy bound and the uniform
convergence together prevent it.

\begin{lemma}\label{lem:null-zero-set}
  Let $(u_n)\subset X$ be such that $\int_0^{e_0}\dd e/u_n\le C$ for every
  $n$ (in particular each $1/u_n$ is integrable, so that
  $\abs{\{u_n=0\}}=0$), and assume that $u_n\to u_\infty$ uniformly on
  $[0,e_0]$. Then the set $Z:=\{x\in[0,e_0]:u_\infty(x)=0\}$ has
  Lebesgue measure zero.
\end{lemma}

\begin{proof}
  First of all $Z$ is measurable, since it is closed ($u_\infty$ is
  continuous). Let $\eps_n:=\max_{[0,e_0]}\abs{u_n-u_\infty}$, so that
  $\eps_n\to0$ by uniform convergence. Fix $M>0$ and choose one $n$ such
  that $\eps_n<1/M$. Then for every $x\in Z$
  \[
    u_n(x)\le u_\infty(x)+\eps_n=\eps_n<\frac1M\,.
  \]
  Since $\abs{\{u_n=0\}}=0$, at almost every point of $Z$ the function
  $1/u_n$ is defined and, by the previous estimate, $1/u_n>M$ there.
  Then we have
  \[
    C\;\ge\;\int_0^{e_0}\frac{\dd e}{u_n}\;\ge\;\int_Z\frac{\dd e}{u_n}
    \;\ge\;M\,\abs Z\,,
  \]
  that is $\abs Z\le C/M$. Since $M>0$ is arbitrary, $\abs Z=0$.
\end{proof}

\begin{proposition}\label{prop:existence}
  $\inf_XF<+\infty$, and there exists $u_\infty\in X$ such that
  $F(u_\infty)=\inf_XF$.
\end{proposition}

\begin{proof}
  \emph{The infimum is finite.} By Proposition~\ref{prop:cycloid}(b),
  $u^{*}$ is a competitor of finite energy, so
  $m:=\inf_XF\le F(u^{*})<+\infty$.

  \emph{A minimizing sequence.} Let $(u_n)\subset X$ with $F(u_n)\to m$;
  discarding finitely many terms, $F(u_n)\le C:=m+1$ for every $n$, that
  is, we are in a sublevel set of $F$. Since both terms of $F$ are
  nonnegative,
  \begin{equation}\label{eq:energy-bound}
    \frac\lambda2\int_0^{e_0}\abs{u_n'}^{2}\,\dd e\le C\,,\qquad
    \int_0^{e_0}\frac{\dd e}{u_n}\le C\,,
  \end{equation}
  and by \eqref{eq:finite-energy} each $1/u_n$ is defined almost
  everywhere.

  \emph{Compactness.} From \eqref{eq:energy-bound},
  $\norm{u_n'}_{L^{2}}\le\sqrt{2C/\lambda}$: the derivatives are bounded
  in the Hilbert space $\Ltwo$, so by the weak compactness of its balls
  \cite[Theorem~3.18]{brezis2011} there is a subsequence, not relabeled,
  with $u_n'\rightharpoonup v_\infty$ weakly in $\Ltwo$. In one dimension
  Sobolev functions are H\"older: by $u_n(x)-u_n(y)=\int_y^xu_n'$ and the
  Cauchy--Schwarz inequality,
  \[
    \abs{u_n(x)-u_n(y)}\le\norm{u_n'}_{L^{2}}\,\abs{x-y}^{1/2}
    \le\sqrt{2C/\lambda}\,\abs{x-y}^{1/2}\,,
  \]
  so the $u_n$ are equi-H\"older of exponent $\half$, in particular
  equicontinuous; and since $u_n(0)=0$,
  $\abs{u_n(x)}=\abs{u_n(x)-u_n(0)}\le\sqrt{2Ce_0/\lambda}$, so they are
  equibounded. These are exactly the hypotheses of the Ascoli--Arzel\`a
  theorem \cite[Theorem~4.25]{brezis2011}: up to a further subsequence,
  $u_n\to u_\infty$ uniformly on $[0,e_0]$. By the definition of weak
  derivative, for every $\psi\in C_c^{\infty}((0,e_0))$,
  \[
    \int_0^{e_0}u_n\psi'\,\dd e=-\int_0^{e_0}u_n'\psi\,\dd e\,,
  \]
  and passing to the limit---uniform convergence on the left, weak
  convergence on the right---we get $\int u_\infty\psi'=-\int v_\infty\psi$,
  i.e.\ $u_\infty\in\Hone$ with $u_\infty'=v_\infty$. Summing up, the
  sublevel sets of $F$ are compact with respect to the convergence
  \begin{equation}\label{eq:convergence}
    u_n\to u_\infty\ \hbox{uniformly on $[0,e_0]$},\qquad
    u_n'\rightharpoonup u_\infty'\ \hbox{weakly in }\Ltwo\,.
  \end{equation}
  Moreover $u_\infty(0)=\lim_nu_n(0)=0$, since uniform convergence implies
  pointwise convergence, and $u_\infty\ge0$ as a uniform limit of
  nonnegative functions: thus $u_\infty\in X$.

  \emph{Lower semicontinuity.} It remains to show that
  $F(u_\infty)\le\liminf_nF(u_n)$ along \eqref{eq:convergence}; we
  consider the two terms separately. For the Dirichlet term we simply
  use the weak lower semicontinuity of the norm in a Hilbert space:
  $\norm{u_\infty'}^{2}=\lim_n\langle u_n',u_\infty'\rangle
  \le\liminf_n\norm{u_n'}\,\norm{u_\infty'}$, whence
  \[
    \int_0^{e_0}\abs{u_\infty'}^{2}\,\dd e
    \le\liminf_n\int_0^{e_0}\abs{u_n'}^{2}\,\dd e\,.
  \]
  For the potential term, the functions $f_n=1/u_n$ are defined almost
  everywhere on $(0,e_0)$, measurable and nonnegative, so Fatou's lemma
  \cite[Lemma~4.1]{brezis2011} applies:
  \begin{equation}\label{eq:fatou}
    \int_0^{e_0}\liminf_{n\to\infty}\frac1{u_n}\,\dd e
    \le\liminf_{n\to\infty}\int_0^{e_0}\frac{\dd e}{u_n}\,.
  \end{equation}
  The left-hand side has to be compared with $\int\dd e/u_\infty$, and
  if $u_\infty$ vanished on a set of positive measure, \eqref{eq:fatou}
  would carry no information there. But this does not happen: by
  Lemma~\ref{lem:null-zero-set}, applied to the sequence $(u_n)$, which
  satisfies \eqref{eq:energy-bound} and converges uniformly, the set
  $Z=\{u_\infty=0\}$ has measure zero. At every $x\in(0,e_0)\setminus Z$
  we have $u_n(x)\to u_\infty(x)>0$, hence $1/u_n(x)\to1/u_\infty(x)$;
  therefore, the integrand being nonnegative,
  \[
    \int_0^{e_0}\liminf_{n\to\infty}\frac1{u_n}\,\dd e
    \ge\int_{(0,e_0)\setminus Z}\liminf_{n\to\infty}\frac1{u_n}\,\dd e
    =\int_0^{e_0}\frac{\dd e}{u_\infty}\,,
  \]
  where $1/u_\infty$ is defined almost everywhere because $\abs Z=0$.
  Together with \eqref{eq:fatou} this says that $1/u_\infty$ is
  integrable, with $\int\dd e/u_\infty\le\liminf_n\int\dd e/u_n$: the
  lower semicontinuity of the potential term.

  \emph{Conclusion.} Adding the two inequalities,
  $F(u_\infty)\le\liminf_nF(u_n)=m$; and $F(u_\infty)\ge m$ because
  $u_\infty\in X$. Hence $F(u_\infty)=m$.
\end{proof}

\subsection{Uniqueness}\label{sub:uniqueness}

The functional $F$ is strictly convex on $X$, hence the minimizer is
unique. In detail: $X$ is convex, being defined by linear conditions and
by the inequality $u\ge0$. Let $u\ne v$ in $X$ with $F(u),F(v)<+\infty$,
so that by \eqref{eq:finite-energy} both $1/u$ and $1/v$ are defined
almost everywhere, and let $w=\half(u+v)$; then
$\{w=0\}\subseteq\{u=0\}$ is a null set, and at almost every point of
$(0,e_0)$
\[
  \frac\lambda2\,\abs{w'}^{2}\le\frac12\Bigl(\frac\lambda2\abs{u'}^{2}
  +\frac\lambda2\abs{v'}^{2}\Bigr)\,,\qquad
  \frac1w\le\frac12\Bigl(\frac1u+\frac1v\Bigr)\,,
\]
by the convexity of $p\mapsto\frac\lambda2p^{2}$ on $\R$ and of
$s\mapsto1/s$ on $(0,+\infty)$. The second function is strictly convex,
so the second inequality is strict wherever $u(x)\ne v(x)$; and
$\{u\ne v\}$ is a nonempty open set, $u$ and $v$ being continuous, hence
of positive measure. Integrating, $F(w)<\half F(u)+\half F(v)$. If $u$
and $v$ were two minimizers, this would give $F(w)<\inf_XF$, which is
absurd. From now on $u_0$ denotes \emph{the\/} minimizer of $F$ on $X$.

\subsection{Monotonicity and positivity}\label{sub:positivity}

Nothing in the definition of $X$ prevents $u$ from vanishing at
interior points (an isolated zero with a $\abs{e-\bar e}^{2/3}$
touchdown even has finite energy). The minimizer does not: it is
nondecreasing, hence positive away from the origin, and this is what
makes the Euler--Lagrange equation available.

\begin{proposition}\label{prop:positive}
  The minimizer $u_0$ is nondecreasing on $[0,e_0]$ and $u_0>0$ on
  $(0,e_0]$. In particular, for every $a\in(0,e_0)$,
  $\delta(a):=\min_{[a,e_0]}u_0>0$.
\end{proposition}

\begin{proof}
  Suppose that $u_0$ is not nondecreasing: there are $x_1<x_2$ in
  $[0,e_0]$ with $u_0(x_2)<u_0(x_1)=:\delta$. Then $\delta>0$, so
  $x_1>0$. We replace $u_0$ by the constant $\delta$ on an interval
  around $x_2$; for the modified function to be continuous, hence in
  $\Hone$, the interval must be chosen so that $u_0=\delta$ at its
  endpoints. Let $(\alpha,\beta)$ be the connected component of the
  relatively open set $\{x\in[0,e_0]:u_0(x)<\delta\}$ that contains
  $x_2$ (with $\beta=e_0$ included if the component reaches $e_0$).
  Since $x_1<x_2$ does not belong to this set, $\alpha\ge x_1>0$, and by
  continuity
  \[
    u_0(\alpha)=\delta\,,\qquad u_0(\beta)=\delta\ \hbox{ if }\beta<e_0\,,
    \qquad u_0<\delta\ \hbox{ on }(\alpha,\beta)\,.
  \]
  Define $v:=u_0$ on $[0,\alpha]\cup[\beta,e_0]$ and $v:=\delta$ on
  $[\alpha,\beta]$. Then $v$ is continuous, and $v\in\Hone$ with
  $v'=u_0'$ outside $[\alpha,\beta]$ and $v'=0$ inside, because $v$ is
  absolutely continuous on each of the three pieces (integrating by
  parts on each piece against a test function, the boundary terms cancel
  by the continuity of $v$). Moreover $v\ge0$ and $v(0)=u_0(0)=0$,
  because $\alpha>0$; so $v\in X$, and $F(v)<+\infty$. Since $v=u_0$
  outside $(\alpha,\beta)$,
  \[
    F(u_0)-F(v)=\int_\alpha^\beta\frac\lambda2\abs{u_0'}^{2}\,\dd e
    +\int_\alpha^\beta\Bigl(\frac1{u_0}-\frac1\delta\Bigr)\dd e\;>\;0\,:
  \]
  the first integral is $\ge0$, and the integrand of the second is
  positive almost everywhere on $(\alpha,\beta)$, where $0<u_0<\delta$
  almost everywhere by \eqref{eq:finite-energy}. This contradicts the
  minimality of $u_0$; so $u_0$ is nondecreasing. If now $u_0(x_0)=0$
  for some $x_0\in(0,e_0]$, then $u_0\equiv0$ on $[0,x_0]$, a set of
  positive measure, against \eqref{eq:finite-energy}; so $u_0>0$ on
  $(0,e_0]$, and its minimum on the compact interval $[a,e_0]$ is
  positive.
\end{proof}

\subsection{The Euler--Lagrange equation}\label{sub:EL}

\begin{proposition}\label{prop:EL}
  The minimizer satisfies $u_0\in C^{\infty}((0,e_0])$,
  \[
    \lambda\,u_0''=-\,\frac1{u_0^{2}}\ \hbox{ on }(0,e_0]\,,\qquad u_0'(e_0)=0\,.
  \]
\end{proposition}

\begin{proof}
  Let $\psi\in C^{\infty}([0,e_0])$ vanish on $[0,a]$ for some
  $a\in(0,e_0)$. By Proposition~\ref{prop:positive}, $u_0\ge\delta(a)>0$
  on $[a,e_0]$, so for $\abs t$ small $u_0+t\psi\ge\delta(a)/2$ on
  $[a,e_0]$, while $u_0+t\psi=u_0$ on $[0,a]$: hence $u_0+t\psi\in X$,
  $F(u_0+t\psi)<+\infty$, and $t\mapsto F(u_0+t\psi)$ is differentiable
  at $t=0$, the derivative being computed under the integral sign (the
  $t$-derivative of the integrand, $\lambda(u_0'+t\psi')\psi'
  -\psi/(u_0+t\psi)^{2}$, is bounded by an integrable function for $\abs t$
  small). Since $u_0$ is a minimizer, this derivative vanishes:
  \begin{equation}\label{eq:weak-EL}
    \int_0^{e_0}\Bigl(\lambda u_0'\psi'-\frac{\psi}{u_0^{2}}\Bigr)\dd e=0\,.
  \end{equation}
  Taking $\psi\in C_c^{\infty}((0,e_0))$, \eqref{eq:weak-EL} says that
  $\lambda u_0'$ is weakly differentiable on $(0,e_0)$ with weak
  derivative $-1/u_0^{2}$, a function which is continuous on $(0,e_0]$ by
  Proposition~\ref{prop:positive}. Hence $u_0'\in H^{1}((a,e_0))$ for every
  $a>0$, so $u_0'$ is continuous on $(0,e_0]$
  \cite[Theorem~8.2]{brezis2011} and $u_0\in C^{1}((0,e_0])$; then
  $u_0'$, being the integral of a continuous function, is $C^{1}$, so
  $u_0\in C^{2}((0,e_0])$ and $\lambda u_0''=-1/u_0^{2}$ there. The
  equation then bootstraps: its right-hand side is $C^{2}$, hence
  $u_0\in C^{4}$, and so on; $u_0\in C^{\infty}((0,e_0])$. Finally, take
  $\psi$ vanishing on $[0,a]$ with $\psi(e_0)=1$ and integrate
  \eqref{eq:weak-EL} by parts on $[a,e_0]$, where $u_0$ is $C^{2}$: the
  integral terms cancel against the equation, and what remains is
  $\lambda u_0'(e_0)\psi(e_0)=0$, i.e.\ $u_0'(e_0)=0$.
\end{proof}

\subsection{Identification with the cycloid, and end of the proof}
\label{sub:identification}

We can now carry out, for the minimizer, the computation announced in
Section~\ref{sec:cycloid}. Put $A:=u_0(e_0)$, which is positive by
Proposition~\ref{prop:positive}. Multiplying $\lambda u_0''=-1/u_0^{2}$
by $u_0'$ on $(0,e_0]$ we get
$\bigl(\frac\lambda2u_0'^{2}-1/u_0\bigr)'=0$, so
$\frac\lambda2u_0'^{2}-1/u_0$ is constant on $(0,e_0]$, equal to its
value at $e_0$, which is $-1/A$ by Proposition~\ref{prop:EL}:
\begin{equation}\label{eq:first-integral-u0}
  \frac\lambda2\,u_0'^{2}=\frac1{u_0}-\frac1A\quad\hbox{on }(0,e_0]\,;
\end{equation}
in particular $u_0\le A$. Since $u_0''<0$ on $(0,e_0]$, the derivative
$u_0'$ is strictly decreasing there, and $u_0'(e_0)=0$ gives $u_0'>0$ on
$(0,e_0)$; so $u_0$ is strictly increasing on $[0,e_0]$ (at $0$ by
continuity), a homeomorphism of $[0,e_0]$ onto $[0,A]$, and
\eqref{eq:first-integral-u0} with the positive root reads
$u_0'=\sqrt{2(A-u_0)/(\lambda Au_0)}$. Let $g\colon[0,A]\to[0,e_0]$ be the
inverse of $u_0$: it is continuous and strictly increasing, $g(0)=0$,
$g(A)=e_0$, and by the inverse function theorem it is $C^{1}$ on $(0,A)$
with
\[
  g'(u)=\frac1{u_0'(g(u))}=\sqrt{\frac{\lambda Au}{2(A-u)}}\,,\qquad 0<u<A\,.
\]
The right-hand side is positive and integrable on $(0,A)$ (it is
bounded near $0$ and behaves like $(A-u)^{-1/2}$ near $A$); so from
$g(u)-g(u_1)=\int_{u_1}^ug'$ for $0<u_1<u<A$, letting $u_1\downarrow0$
and then $u\uparrow A$ and using the continuity of $g$ at $0$ and $A$,
\begin{equation}\label{eq:g}
  g(u)=\int_0^u\sqrt{\frac{\lambda As}{2(A-s)}}\,\dd s\quad\hbox{for }0\le u\le A\,,
  \qquad\hbox{in particular}\quad e_0=g(A)\,.
\end{equation}
The integral is computed by the substitution $s=A\sin^{2}\phi$,
$0\le\phi\le\pi/2$, for which $\dd s=2A\sin\phi\cos\phi\,\dd\phi$ and
$\sqrt{\lambda As/(2(A-s))}=\sqrt{\lambda A/2}\,\tan\phi$: for
$0\le\theta\le\pi/2$,
\[
  g(A\sin^{2}\theta)=\sqrt{2\lambda}\,A^{3/2}\int_0^\theta\sin^{2}\phi\,\dd\phi
  =\sqrt{\frac\lambda2}\,A^{3/2}\,(\theta-\sin\theta\cos\theta)\,.
\]
At $\theta=\pi/2$ this is \eqref{eq:g}:
$e_0=\sqrt{\lambda/2}\,A^{3/2}\,\pi/2$, that is
$A^{3/2}=2\sqrt2\,e_0/(\pi\sqrt\lambda)$ and
$A=\lambda^{-1/3}(2\sqrt2\,e_0/\pi)^{2/3}$, the constant of
\eqref{eq:cycloid}; with this value $\sqrt{\lambda/2}\,A^{3/2}=2e_0/\pi$
and the last display reads $g(U(\theta))=\varepsilon(\theta)$ in the
notation of Proposition~\ref{prop:cycloid}, i.e.\
$u_0(\varepsilon(\theta))=U(\theta)$ for every $\theta\in[0,\pi/2]$.
Hence $u_0=U\circ\varepsilon^{-1}=u^{*}$.

\begin{proof}[Proof of Theorem~\ref{thm:main}]
  By Proposition~\ref{prop:existence} and Section~\ref{sub:uniqueness},
  $F$ has exactly one minimizer $u_0$ on $X$, and we have just shown that
  $u_0$ is the function $u^{*}$ defined by \eqref{eq:cycloid}. Its
  continuity, strict monotonicity and smoothness on $(0,e_0]$ are in
  Proposition~\ref{prop:cycloid}(a); equation \eqref{eq:EL} is
  Proposition~\ref{prop:EL} together with $u^{*}(0)=0$.
\end{proof}

\begin{proof}[Proof of Corollaries~\ref{cor:cusp} and~\ref{cor:values}]
  Proposition~\ref{prop:cycloid}(c) and (b), respectively; in (b),
  $T(u^{*})=\int\dd e/u^{*}=2e_0/A=\lambda^{1/3}(\pi^{2}e_0)^{1/3}$
  because $A=\lambda^{-1/3}\cdot2(e_0/\pi)^{2/3}$, and
  $\frac\lambda2\int\abs{u^{*\prime}}^{2}=e_0/A=\half T(u^{*})$.
\end{proof}

\begin{remark}[the clock along the minimizer]\label{rem:clock}
  Along the risk dynamics $\dot e=-u^{*}(e)$, $e(0)=e_0$, the parameter
  $\theta$ of \eqref{eq:cycloid} decreases linearly in time. Indeed
  $\dd t=-\dd e/u^{*}(e)$ and, by \eqref{eq:eps-prime},
  $\dd t=-\varepsilon'(\theta)\,\dd\theta/U(\theta)=-(4e_0/\pi A)\,\dd\theta$:
  the factor $\sin^{2}\theta$ cancels, and $\theta$ runs uniformly from
  $\pi/2$ to $0$ in the time $(\pi/2)(4e_0/\pi A)=2e_0/A=T(u^{*})$. Thus
  $e(t)=\varepsilon\bigl(\tfrac\pi2(1-t/T(u^{*}))\bigr)$ in closed form,
  and the risk reaches zero at $t=T(u^{*})$ with $\dot e=0$. Since
  $\sin^{2}\theta$ has mean $\half$ over $[0,\pi/2]$, the cycloid takes
  exactly twice the time of the constant schedule running at its own
  peak rate $A$: this factor $2$ is the price of the landing condition.
\end{remark}

\section{Consistency with the H\"older regime}\label{sec:holder}

The paper \cite{betti2026holder} classifies the power-law schedules
$u(e)=\eta_0e^{\alpha}$ by what they do to the weights near a
nondegenerate minimum $w^{*}$ of $E$, where $\abs{\nabla E}^{2}\asymp e$.
Along \eqref{eq:dynamics} the speed of the weights is
$\abs{\dot w}=u(e)/\abs{\nabla E}\asymp e^{\alpha-1/2}$, and in terms of
$r=\abs{w-w^{*}}$ the terminal dynamics is $\dot r\asymp-r^{2\alpha-1}$.
Four regimes result: for $\alpha<\half$ the speed blows up at arrival
(singular regime); for $\alpha=\half$ it stays bounded and positive
(critical regime); for $\half<\alpha<1$ the arrival is in finite time
with vanishing speed, the terminal vector field being H\"older
continuous of exponent $2\alpha-1$ but not Lipschitz (\emph{H\"older
regime}); for $\alpha\ge1$ the field is Lipschitz and the arrival is only
asymptotic. Note that on the risk itself every $\alpha\in(0,1)$ lands
tangentially, $e(t)\propto(t^{*}-t)^{1/(1-\alpha)}$; the threshold $\half$
is invisible in the risk and is created by the flat bottom of the bowl,
$e\asymp r^{2}$, which halves the exponent.

Theorem~\ref{thm:main} is consistent with this picture in three
respects. First, the minimizer is not a power law, but at the cusp
$u^{*}(e)\sim(9/(2\lambda))^{1/3}e^{2/3}$ (Corollary~\ref{cor:cusp}), and
$2/3$ is the exponent obtained in \cite{betti2026holder} by substituting
the ansatz $\eta_0e^{\alpha}$ in the Euler--Lagrange equation
(Remark~\ref{rem:power-law}). The variational problem thus selects,
among all schedules, one whose terminal behaviour is that of
$\alpha=2/3\in(\half,1)$: arrival in the finite time
$\lambda^{1/3}(\pi^{2}e_0)^{1/3}$ with weight speed
$\abs{\dot w}\asymp e^{1/6}\to0$. The terminal field in $r$ has H\"older
exponent $2\alpha-1=1/3$, which happens to be the exponent of Zak's
canonical terminal attractor $\dot x=-x^{1/3}$ \cite{zak1988}. Second,
the finite-energy condition alone already excludes the other regimes for
power laws (Remark~\ref{rem:powers}): the Dirichlet term excludes
$\alpha\le\half$ and the time term excludes $\alpha\ge1$. The H\"older
window is not a consequence of optimality; it is the set of power-law
schedules the functional can see at all, and optimality picks a point
inside it. Third, the selection is robust to the weights: $\lambda$ only
rescales the rate by $\lambda^{-1/3}$ and the time by $\lambda^{1/3}$,
leaving the cusp exponent untouched (Remark~\ref{rem:scaling}).

\begin{remark}[power laws]\label{rem:powers}
  Let $v_\alpha(e)=e^{\alpha}$ with $\alpha>0$. On $(0,e_0)$ it is smooth
  with $v_\alpha'=\alpha e^{\alpha-1}$, and this derivative is square
  integrable exactly when $2\alpha-2>-1$; so $v_\alpha\in X$ if and only
  if $\alpha>\half$. Its zero set is $\{0\}$, and
  $\int_0^{e_0}\dd e/v_\alpha=\int_0^{e_0}e^{-\alpha}\,\dd e$ is finite
  exactly when $\alpha<1$. Hence, whatever $\lambda$,
  \begin{equation}\label{eq:window}
    F(v_\alpha)<+\infty\iff\tfrac12<\alpha<1\,.
  \end{equation}
  For $\lambda=1$ and $\alpha=2/3$,
  $F(v_{2/3})=\int_0^{e_0}\bigl(\frac12\cdot\frac49e^{-2/3}+e^{-2/3}\bigr)\dd e
  =\frac{11}3\,e_0^{1/3}$; the value of the minimizer is
  $\frac32\pi^{2/3}e_0^{1/3}\approx3.218\,e_0^{1/3}$, smaller, as it must
  be. The linear profile $v_1(e)=e$, the schedule of Polyak's step, has
  $F(v_1)=+\infty$: it takes infinite time. Constants $u\equiv c$ are
  not in $X$ at all, since $u(0)\ne0$.
\end{remark}

\begin{remark}[the pure power law]\label{rem:power-law}
  The function $\bar v(e)=(9/(2\lambda))^{1/3}e^{2/3}$ is an exact
  solution of $\lambda u''=-1/u^{2}$ on $(0,+\infty)$: indeed
  $\lambda(ce^{2/3})''=-\frac29\lambda ce^{-4/3}$ equals $-c^{-2}e^{-4/3}$
  exactly when $c^{3}=9/(2\lambda)$. It is the solution found in
  \cite{betti2026holder} by the ansatz $u=\eta_0e^{\alpha}$. But it does
  not satisfy $u'(e_0)=0$, so it is not the minimizer on a finite
  interval: the minimizer follows it at the cusp and bends over to reach
  its maximum at $e_0$. One can check that $\bar v$ arrives sooner,
  $T(\bar v)=6^{1/3}\lambda^{1/3}e_0^{1/3}\approx1.817\,\lambda^{1/3}e_0^{1/3}$
  against $T(u^{*})\approx2.145\,\lambda^{1/3}e_0^{1/3}$, but pays more
  in the penalty: $F(\bar v)=2\cdot6^{1/3}\lambda^{1/3}e_0^{1/3}
  \approx3.634\,\lambda^{1/3}e_0^{1/3}$, against
  $F(u^{*})\approx3.218\,\lambda^{1/3}e_0^{1/3}$. The same equation
  $u''=-c/u^{2}$ governs, in one dimension, the Alt--Phillips functional
  with negative exponent studied by De Silva and Savin
  \cite{desilva2022}, whose potential $u^{-1}\chi_{\{u>0\}}$ charges
  nothing on the zero set; there dead zones are free, here they cost
  infinitely much, and the two problems are genuinely different.
\end{remark}

\begin{remark}[scaling]\label{rem:scaling}
  For $\lambda_1,\lambda_2>0$ the functional
  $\int_0^{e_0}(\lambda_1\abs{u'}^{2}/2+\lambda_2/u)\,\dd e$ of
  \cite{betti2026holder}, without the amplitude term, is $\lambda_2F$
  with $\lambda=\lambda_1/\lambda_2$: a single parameter suffices, and it
  can be scaled away as well. Put $b=\sqrt\lambda$ and $v(s)=u(bs)$ for
  $0<s<e_0/b$; then $u'(e)=v'(e/b)/b$, and the substitution $e=bs$ gives
  \[
    F(u)=\frac\lambda{b}\int_0^{e_0/b}\frac12\abs{v'}^{2}\,\dd s
    +b\int_0^{e_0/b}\frac{\dd s}{v}
    =\sqrt\lambda\;F_1(v)\,,
  \]
  where $F_1$ is the functional \eqref{eq:F} with $\lambda=1$ on the
  interval $(0,e_0/b)$. Minimizers correspond, and this is consistent
  with Theorem~\ref{thm:main}: the cycloid for $(\lambda,e_0)$ is
  obtained from the cycloid for $(1,e_0/\sqrt\lambda)$ by the change of
  variable $e\mapsto e/\sqrt\lambda$, which leaves the parametrization
  $e(\theta)$ in \eqref{eq:cycloid} unchanged and divides $u^{*}$ by
  $\lambda^{1/3}$. The amplitude term $\lambda_3\int u^{p}/p$ of
  \cite{betti2026holder} is not covered by Theorem~\ref{thm:main}: the
  existence argument of Section~\ref{sub:existence} goes through
  unchanged, since that term is nonnegative and passes to the limit
  under uniform convergence, but the identification step has no closed
  form, and what the amplitude term does to the cusp exponent we leave
  open.
\end{remark}

\begin{remark}
  The two claims of Theorem~\ref{thm:main} that matter for the
  algorithm are the finite arrival time and the vanishing rate at
  arrival, $u^{*}(0)=0$ with $u^{*}\sim ce^{2/3}$. The first says that
  the schedule is a terminal attractor for the risk; the second, combined
  with the classification above, says that it lands the weights softly.
  What is not proved here, and not claimed, is anything about the
  discretization of \eqref{eq:dynamics} along $u^{*}$: the minimum-time
  property is a statement in the time of the equation, not in the number
  of gradient evaluations, and the cost of following $u^{*}$ numerically
  depends on the curvature of $E$ near the minimum. This is the subject
  of separate work.
\end{remark}

\section*{Acknowledgments}
The authors thank Giovanni Bellettini and Genni Fragnelli for helpful
discussions that clarified fundamental aspects of the proof of the
existence and uniqueness theorem. The authors used an AI assistant
(Claude, Anthropic) for help with the writing and typesetting of this
note; all mathematical content was checked by the authors, who take full
responsibility for it.

\bibliographystyle{amsplain}
\bibliography{biblio}

\end{document}